\documentclass{article}

\usepackage{microtype}
\usepackage{graphicx}
\usepackage{subfigure}
\usepackage{booktabs} 

\usepackage{hyperref}

\usepackage[accepted]{icml2023}

\usepackage{amsmath}
\usepackage{amssymb}
\usepackage{mathtools}
\usepackage{amsthm}
\usepackage{multirow}

\usepackage[capitalize,noabbrev]{cleveref}

\theoremstyle{plain}
\newtheorem{theorem}{Theorem}[section]
\newtheorem{proposition}[theorem]{Proposition}

\theoremstyle{definition}
\newtheorem{definition}[theorem]{Definition}

\theoremstyle{remark}

\usepackage[textsize=tiny]{todonotes}

\icmltitlerunning{Population Expansion for Training Language Models with Private Federated Learning}

\begin{document}

\twocolumn[
\icmltitle{Population Expansion for Training Language Models with \\ Private Federated Learning}



\icmlsetsymbol{equal}{*}

\begin{icmlauthorlist}
\icmlauthor{Tatsuki Koga}{intern}
\icmlauthor{Congzheng Song}{comp}
\icmlauthor{Martin Pelikan}{comp}
\icmlauthor{Mona Chitnis}{comp}
\end{icmlauthorlist}

\icmlaffiliation{comp}{Apple}
\icmlaffiliation{intern}{UC San Diego, work done while interning at Apple}

\icmlcorrespondingauthor{Congzheng Song}{csong4@apple.com}

\icmlkeywords{Machine Learning, ICML}

\vskip 0.3in
]



\printAffiliationsAndNotice{}  

\begin{abstract}
Federated learning (FL) combined with differential privacy (DP) offers machine learning (ML) training with distributed devices and with a formal privacy guarantee.
With a large population of devices, FL with DP produces a performant model in a timely manner. 
However, for applications with a smaller population, not only does the model utility degrade as the DP noise is inversely proportional to population, but also the training latency increases since waiting for enough clients to become available from a smaller pool is slower.
In this work, we thus propose expanding the population based on domain adaptation techniques to speed up the training and improves the final model quality when training with small populations.
We empirically demonstrate that our techniques can improve the utility by 13\% to 30\% on real-world language modeling datasets.
\end{abstract}

\section{Introduction}
Federated learning (FL) \citep{mcmahan_communication-efficient_2017} enables training machine learning (ML) models using on-device data and is widely used in our daily lives as usage of mobile devices, e.g., smartphones, smart watches, and smart speakers, increases.
Although FL, by design, does not require raw data to be transmitted from devices, privacy breaches can happen by transmitting model gradients to the central server.
Thus, FL algorithms are modified to satisfy differential privacy (DP) \citep{mcmahan_learning_2018-1} to provide a formal privacy guarantee.
We refer this learning framework as private federated learning (PFL).

Successful ML models trained with PFL typically require the number of devices sampled at each round, \textit{cohort size}, to be large enough to reduce the detrimental impact of DP noise on the model utility~\cite{anil-etal-2022-large}. 
The requirement of large cohort size, which is easily met with hundreds of millions of devices, can be hard to fulfill for applications with device-constrained populations. 
For a motivating example, to train a language model (LM) with PFL for automatic speech recognition (ASR) system in a virtual assistant, the on-device training data are transcribed speech. 
For popular languages, such as English or  Chinese, there are ample devices with transcriptions. 
However, for less popular languages such as Romanian or Swahili, the population with data is orders of magnitude smaller due to the limited speaker base.
In such small populations, as we will show in Section~\ref{sec:sample}, the server needs to spend much longer waiting for a full cohort of devices to become available in each iteration, which is impractical for models that require thousands of iterations to converge. 
Thus, PFL has the tradeoff among privacy, utility, and \emph{latency} for device-constrained applications.

\paragraph{Our contributions} 
In this work, we develop approaches to expand the population size to address the latency bottleneck for PFL in the device-constrained scenarios.
We propose to use data from different applications than the target application to augment the training data, e.g. there are more devices with typed text than those with audio transcriptions as the messaging application is used more frequently than a virtual assistant.
Population expansion for PFL has three benefits: (1) training will be faster as there are more devices available, (2) DP noise scale will be smaller from amplification by subsampling~\cite{wang2019subsampled} by making population size larger, and (3) sampling error will be smaller.
We explore combinations of various domain adaptation techniques and show that they outperform naively augmenting the devices from other sources.
We focus on training LMs and evaluate the proposed approaches on public benchmark datasets including Reddits Comments and Common Voice. 
We demonstrate our methods can expand the population size by 10 times, which significantly reduces the latency and achieves better model utility. 

\subsection{Related Work}
Prior works on domain adaptation in the LM applications focuses on centralized training.
\citet{jiang-zhai-2007-instance} explored instance weighting with importance sampling to reweight the training objective for domain adaptation.
\citet{moore_intelligent_2010} selected and used a portion of non-domain-specific language data for domain-specific LM training.
\citet{moriokal_language_2018} extended LM neural networks (NNs) to have domain-specific and domain-shared representations so that those representations are learned separately.
\citet{gururangan_demix_2021} focused on transformer model and modify the model architecture to have domain-specific layers.
More recently, \citet{chronopoulou_efficient_2022-1} adopted hierarchical network structures for training on data from a larger number of domains, where models are gradually trained along with the hierarchy in a top-down manner.

With regards to domain adaptation in the federated setting, prior works address the setting where the clients and the server own data from different domains \citep{peng_federated_2019-1, yao_federated_2022}.
\citet{shen_fedmm_2021} extended the adversarial domain adaptation technique to the federated setting, but their main focus is cross-silo FL, where the number of clients is much smaller.
\citet{peterson_private_2019} also proposed a domain adaptation technique in cross-silo FL with differential privacy, which properly combines general and specific models.


\section{Preliminaries}
\paragraph{Federated Learning (FL)}~\citep{mcmahan_communication-efficient_2017} enables model training on multiple devices, each having a separate dataset, without sharing on-device dataset with a central server.
In particular, we focus on \textit{cross-device} FL where the number of clients is very large, as opposed to cross-silo FL where client population is small.
The standard iterative procedure for training machine learning models executes at each iteration $t$:
(1) the central server samples a set of clients $\mathcal{C}_t$ from the population,
(2) each sampled client $i\in \mathcal{C}_t$ downloads the shared model parameter $\theta_t$ from the server and locally trains the model on its own data to produce a local model $\theta_i$, 
(3) each sampled client $i$ sends back the model difference $\Delta_{t,i} = \theta_i - \theta_t$ to the server, and
(4) the server aggregates the model differences as a ``pseudo-gradient'' $\Delta_t = \frac{1}{|\mathcal{C}_t|} \Delta_{t,i}$ and uses it to update $\theta_t$ with any standard optimizer.

\paragraph{Differential Privacy (DP)} provides strong privacy protections for sensitive data on device. 
DP is formally defined as follows:
\begin{definition}[$(\epsilon, \delta)$-DP \citep{dwork_calibrating_2006}]
A randomized algorithm $M$ satisfies $(\epsilon, \delta)$-DP if for any neighboring datasets $D,D^{\prime}$ and for any $S\subseteq \mathrm{range}(M)$,
\begin{align*}
    \Pr [M(D) \in S] \leq \exp (\epsilon) \Pr [M(D^\prime) \in S] + \delta.
\end{align*}
\end{definition}
We say two datasets $D, D^\prime \in \mathcal{X}$ are neighboring if they differ on at most an individual's participation.
Two additional steps are added to the FL algorithm to ensure a DP guarantee:
(1) each sampled client clips the model difference before sending it back to have a bounded norm, and 
(2) the server applies a DP building block, commonly the Gaussian mechanism \cite{dwork_algorithmic_2014}, when aggregating the model differences to get the \textit{noisy} pseudo-gradient.
We focus on using the Gaussian mechanism for aggregating the model differences in this work.
The noise variance is then calibrated by the moment accountant \cite{abadi_deep_2016, mironov_renyi_2017-2, mironov_renyi_2019} with fixed sampling rate $q$ (fraction of clients sampled in each iteration), number of training iterations $T$, and privacy budgets $(\epsilon, \delta)$.

\section{Expanding Population in PFL}
\subsection{Device Sampling Latency} \label{sec:sample}
We first formulate how  population size $N$ impacts latency in PFL.
In each round of PFL, \textit{cohort size} $C \approx Nq$ of devices are sampled to participate in training, where $q$ is the device sampling probability to provide an amplification on privacy~\cite{wang2019subsampled}.
Server tends to over-sample by using a slightly larger $q>\frac{C}{N}$ to improve the latency. 
In reality, only a proportion of devices satisfying certain conditions (e.g. locked, charging and on Wi-Fi) are eligible for training and devices might dropout or abort training~\cite{bonawitz2019towards,paulik2021federated}, and we denote this ratio of eligible devices as $p$. 
Therefore, if $C$ is larger than $Npq$, we need to wait until enough devices become available to participate before updating the model.

More formally, assume $Npq < C$, we model the latency to wait $C-Npq$ devices more to become available and be sampled as follows. 
Let $m = N-Np$ be the number of current unavailable devices, $k = C-Npq$ be the number of devices needed for current PFL iteration. 
\begin{proposition}
Assume that the time for the $i$-th unavailable device becoming available and being sampled for training is $T_i\sim \mathrm{Exponential} (\lambda)$. 
Let $U_k$ be the random variable which describes the time when the first $k$ devices become available and are sampled. Then
\begin{equation}
\label{eq:latency}
\frac{1}{\lambda}\cdot\frac{C-Npq}{N(1-p)+1}\leq \mathbb{E}[U_k]  \leq \frac{C}{\lambda(N-C)}.
\end{equation}
\end{proposition}

\begin{figure}
    \centering
    \includegraphics[width=0.48\textwidth]{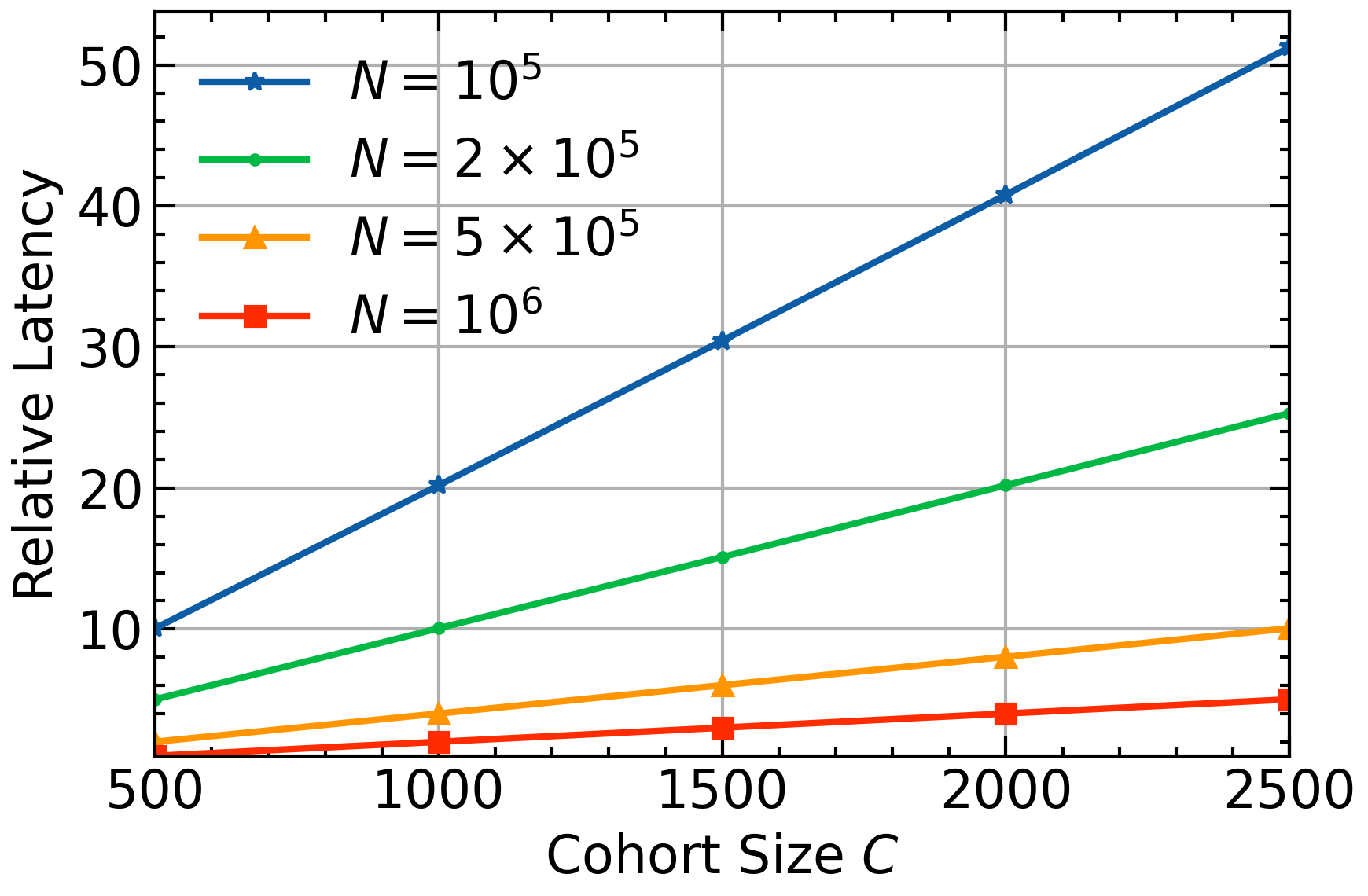}
    \caption{Relative latency estimated with Equation~\ref{eq:latency} for different cohort sizes $C$ and population sizes $N$.}
    \label{fig:latency}
\end{figure}
We defer the proof to Appendix~\ref{sec:appendix}. 
We use exponential time model since it is a common choice for modeling training time in the distributed scenario~\cite{lee2017speeding,tandon2017gradient,nguyen2022federated}. 

From the above proposition we see that the expected latency $U_k$ is inversely proportional to the population size, i.e. the smaller the population size, the longer the server needs to wait for enough devices to become available in each iteration. 
Figure~\ref{fig:latency} illustrates the relationship between latency and population size. 

\subsection{Domain Adaptation for Expanding Population}
The small population situation happens often when building task-specific LMs, where potential data sources are scarce, e.g. training a LM on Swahili spoken texts as a part of virtual assistant system. 
It is a challenging task since only a small number of users are frequent users of a virtual assistant and have Swahili speech on their devices.
Nonetheless, for such device-constrained locales, there could be other data sources, e.g., typed texts, with larger population. 
This motivates us to expand the population by exploiting another text source with a different distribution to train the LM for the target data source, which can be cast as a domain adaptation (DA) problem.
Following DA convention, we denote data from other source applications with larger population as \emph{source domain} $\mathcal{S}$, and data from target application with smaller population as \emph{target domain} $\mathcal{T}$. 

\paragraph{Goal} We wish learn a global model that minimizes the objective $\mathbb{E}_{x\sim \mathcal{T}} [L(x)]$, where $L$ is the loss function, with data from $\mathcal{S}\cup \mathcal{T}$ under a fixed privacy budget $(\epsilon, \delta)$.
The latency-utility trade-off should be much better than training in $\mathcal{T}$ alone.

\paragraph{Instance weighting (IW)}
Naively training with devices sampled from $\mathcal{S}\cup \mathcal{T}$ would bias towards $\mathcal{S}$ due to its larger population. 
To remedy this sampling bias, we apply instance weighting~\citep{jiang-zhai-2007-instance} on the training objective: 
\begin{equation}
    \mathbb{E}_{x\sim \mathcal{S}\cup\mathcal{T}}[w(x) L(x)],
\end{equation}
where $w(x) = p_\mathcal{T}(x) / p_\mathcal{\pi}(x)$ is the importance weight, $\pi \in \{\mathcal{S}, \mathcal{T}\}$ denotes which domain $x$ is from and $p_\mathcal{\pi}(x)$ is the data density function for domain $\pi$.
As $p_\mathcal{\pi}(x)$ has to be estimated privately, we choose to approximate it with unigram likelihood $\hat{p}_\pi(x)=\prod_i \hat{u}_\pi(x_i)$ as unigram frequency $\hat{u}_\pi$ can be efficiently learned with a relative small privacy budget. 

The product of unigrams in $\hat{p}_\pi(x)$ can lead to bipolarized density estimation, and thus unstable importance weights. 
We instead use relative importance weight~\cite{yamada2013relative} to provide a more robust estimation:
\begin{equation}
\label{eq:iw}
    w(x) = \frac{\hat{p}_\mathcal{T}(x)}{\alpha \hat{p}_\mathcal{T}(x) + (1-\alpha) \hat{p}_\pi(x)},
\end{equation}
where $\alpha$ is the proportion of the devices with data from $\mathcal{T}$. 

The overall PFL training procedure with IW is: (1) learn the unigram frequency $\hat{u}_\pi$ for $\pi \in \{\mathcal{S}, \mathcal{T}\}$ with privacy budget $(\epsilon_0, \delta_0)$ which can be done with private federated statistics~\cite{mcmillan2022private}, and (2) train model using objective weighted by Equation~\ref{eq:iw} with privacy budget $(\epsilon-\epsilon_0, \delta-\delta_0)$. 

\paragraph{Pretrain in $\mathcal{S}$ and finetune in $\mathcal{T}$ (PT)} 
Recent work~\cite{ganesh2023public} has shown that pretraining a model in a different domain to target domain with a large population reduces the amount of data required for private finetuning. 
We consider pretraining in $\mathcal{S}$ with a large cohort size $C$ and finetune in $\mathcal{T}$ with a small cohort size $\alpha C$ so that the latency for finetuning stays roughly the same as pretraining. 
We enforce that the population of $\mathcal{S}$ and  $\mathcal{T}$ to be disjoint so that both pretraining in $\mathcal{S}$ and finetuning in $\mathcal{T}$ can spend privacy budget of $(\epsilon, \delta)$ with parallel composition~\cite{mcsherry2009privacy}. 

\paragraph{Instance weighted pretraining (IWPT)} Domain adaptive pretraining~\cite{gururangan-etal-2020-dont} (DAPT) demonstrated the benefits of pretraining with in-domain data.
However because the in-domain population is limited and it is inefficient to train with PFL, we consider instance weighted pretraining on $\mathcal{S}$ with objective weighted by Equation~\ref{eq:iw} as an approximation of DAPT. 

\section{Experiment}
\subsection{Datasets}
To simulate a practical situation, we focus on using real-world datasets with user identifiers so that we can partition data naturally by users.
In particular, we use two sources of data: (1) Reddits \citep{caldas_leaf_2019} and (2) Common Voice (CV) \citep{ardila_common_2020-1} to build two  datasets for DA tasks.
More data processing details are described in Appendix~\ref{app:data}.

\paragraph{SubReddits}
The first constructed DA dataset consists of only the Reddits dataset with different \emph{SubReddit} topics.
We treat a set of similar subreddits as a domain, where we choose stock-related subreddits \{\textit{Superstonk, amcstock, wallstreetbets, GME, Wallstreetsilver}\} as $\mathcal{S}$ and news-related subreddits \{\textit{news, worldnews, politics}\} as $\mathcal{T}$.
As a result of the constrution, we have $117,708$ clients in total and $14,072$ clients (about $12\%$) have target domain data as well as source domain data.

\paragraph{CV\&Reddits} 
The other constructed DA dataset combines Reddit (typed texts) and CV (transcribed audios) which simulates the difference between spoken and typed texts domains.
We treat texts from Reddits as $\mathcal{S}$ and texts from CV as $\mathcal{T}$.
CV dataset has $68,312$ clients.
We randomly select clients from Reddit dataset so that the total number of clients is 10 times more than the number of clients with Common Voice data.

\subsection{Experiment Setup}
Since there usually is a constraint on the client device storage and communication cost in real world applications, we consider a rather simple LSTM following \cite{mcmahan_learning_2018-1}.
We evaluate the performance of our approaches by the perplexity (PPL) in $\mathcal{T}$.
We divide clients into training, validation, and test sets with the ratio of 6:2:2, where the hyper-parameters are tuned on validation set.

We consider two baselines with unweighted objective: (1) training with cohort sizes $\alpha C$ and $C$ in $\mathcal{T}$ only where $\alpha$ is the proportion of the devices with data from $\mathcal{T}$, and (2) training with cohort size $C$ in $\mathcal{S}\cup\mathcal{T}$. 
We also experiment the baseline (2) with domain adaptive layers proposed in domain-shared/domain-specific representations (DSDSR)~\cite{moriokal_language_2018} and DEMix~\cite{gururangan_demix_2021}.

To speed up the training process, we follow~\cite{mcmahan_learning_2018-1} and set the cohort size $C$ to be 5,000 for adjusting the magnitude of noise in the DP analysis and to be 400 for actual training.
We set $\alpha=0.1$ i.e. the ratio of population between $\mathcal{T}$ and $\mathcal{S}$.
All experiments last for 2,000 server iterations and 1 client iteration.
For fine-tuning experiments (PT and IWPT), we split the server iterations into 1,000 and 1,000 for pretraining and fine-tuning, respectively.
We use FedAdam~\cite{reddi_adaptive_2022} as the server optimizer with learning rate 0.1 and SGD as the client optimizer with learning rate 0.5. 

We set the total privacy parameters to $(\epsilon, \delta) = (2, 10^{-6})$ throughout the experiments.
The clipping bound of Gaussian mechanism in PFL is set to 0.5.
For IW and IWPT, we allocate $(\epsilon_0, \delta_0)=(0.8, 0)$ for estimating unigrams with Geometric Mechanism~\cite{ghosh2009universally}, and $(\epsilon,\delta) = (1.2, 10^{-6})$ for model training.
To bound the sensitivity for the unigram estimation, we use at most 5 sequences with each of which have a fixed length of 10 tokens.

\subsection{Results}
\begin{table}[t]
\centering
\begin{tabular}{cc|cc}
\toprule
Dataset & Approach & val PPL & test PPL \\
\midrule
\multirow{8}{*}{SubReddits} & $\mathcal{T}$  w. $\alpha C$ & 415.35 & 414.61\\
 & $\mathcal{T}$  w. $C$ & 358.37 & 358.06\\
 & $\mathcal{S}\cup\mathcal{T}$ & 398.82 & 400.90 \\
  & DSDSR & 379.06 & 380.79 \\
  & DEMix & 395.02 & 396.68 \\
 & IW & 354.81 & 356.37 \\
 & PT & 369.57 & 370.24\\
 & IWPT & \textbf{346.85} & \textbf{347.78} \\
\midrule 
\multirow{8}{*}{CV\&Reddits} & $\mathcal{T}$  w. $\alpha C$ & 302.43 & 320.13\\
  & $\mathcal{T}$  w. $C$ & 215.96 & 241.42\\
 & $\mathcal{S}\cup\mathcal{T}$ & 275.85 & 302.64 \\
  & DSDSR & 206.07 & 233.43 \\
 & DEMix & 226.09 & 255.87 \\
 & IW & 218.61 & 234.22 \\
 & PT & 195.49 & 217.62\\
 & IWPT & \textbf{180.98} & \textbf{203.14} \\
\bottomrule
\end{tabular}
\caption{Perplexity scores for baselines and different DA approaches. We set cohort size $C=5,000$ and $\alpha=0.1$.}
\label{ppl-table}
\end{table}

Table~\ref{ppl-table} summarizes the model performance of our algorithm and baseline approaches.
First, we observe from results on both datasets that training models with a small cohort size $\alpha C$ in $\mathcal{T}$ only has the worst performance, which is because the DP noise dominates the model update in each iteration.
Increasing the cohort size to $C$ can greatly improve the utility for $\mathcal{T}$ only. 
However, according to the argument made in Section~\ref{sec:sample}, we need to trade off a significant amount of training time for a larger $C$.

For the baseline trained with large population size in $\mathcal{S}\cup\mathcal{T}$ and large cohort size, simply treating source domain data as target domain data does not improve the performance much possibly because source domain data is from a different distribution and has larger volume which dominates the model update.
DA specific architectures (DSDSR and DEMix) improved this baseline to some extent but can incur more communication cost due to larger model sizes.

On the other hand, both IW and PT approaches outperform the baseline methods, and are better than the DA specific architectures on SubReddits dataset. 
The combined IWPT approach achieves the best PPL, 13\% and 30\% lower than the baseline models on SubReddits and CV\&Reddits, respectively.

\section{Conclusion and Future Work}
We demonstrate that the population size being small in PFL not only harms the model quality but also slows down the LM training.
With our proposed domain adaptation algorithm, which weights the source domain data appropriately, we show it is possible to have a larger population and train LMs with a better quality in a timely manner.
Since instance weighting framework can be applied to other data domains than languages, extending the framework to other domains, e.g., images, is a direction for future work.


\bibliography{ref}
\bibliographystyle{icml2023}

\newpage
\appendix
\onecolumn


\section{Dataset Preprocessing}
\label{app:data}
The set of known vocabulary is built with target domain data in the training set by choosing top 10K frequent words and is assumed to be known in advance.
Every word outside the vocabulary list is mapped as \texttt{<UNK>}.
We append \texttt{<BOS>} to the beginning and \texttt{<EOS>} to the end of every sentence. 
Within each user, we limit the number of tokens (words) to 1,600 and cut the input sentences into sequences of length 10.
When a sequence has length less than 10, we append \texttt{<PAD>} to make it have length 10.

\section{Proof of Proposition 3.1}
\label{sec:appendix}

Here we restate and prove Proposition 3.1. 
Let $m = N-Np$ be the number of current unavailable devices, $k = C-Npq$ be the number of devices needed for current PFL iteration. 

\begin{proposition}
Assume that the time for the $i$-th unavailable device becoming available and being sampled for training is $T_i\sim \mathrm{Exponential} (\lambda)$. 
Let $U_k$ be the random variable which describes the time when the first $k$ devices become available and are sampled. Then
\begin{equation}
\nonumber
\frac{1}{\lambda}\cdot\frac{C-Npq}{N(1-p)+1}\leq \mathbb{E}[U_k]  \leq \frac{C}{\lambda(N-C)}.
\end{equation}
\end{proposition}

\begin{proof}
We first state two properties about Exponential distribution:
\begin{enumerate}
\item The minimum of $n$ exponential random variables is exponential: $\min\{T_1, \dots T_n\}\sim \mathrm{Exponential} (n\lambda)$~\cite{parzen1960modern}.
\item The exponential random variable $T_i$ is a memoryless: $P(T_i > a + b | T_i > b) = P(T_i > a)$~\cite{parzen1960modern}.
\end{enumerate}
In our definition, $U_1 = \min\{T_1, \dots T_m\}\sim \mathrm{Exponential} (m\lambda)$ from the first property.

WLOG, let $U_i=\min\{T_i, \dots, T_m\} | T_j > U_{i-1}$ where $j=\{i,\dots, m\}$ and $i>1$, then with the second property we can derive:
\begin{equation}
\nonumber
\begin{aligned}
P(U_i - U_{i-1} > a) &=  P(U_i - U_{i-1} > a | U_i > U_{i-1}) \\
&= P(U_i > a + U_{i-1} | U_i > U_{i-1}) \\ 
& = P(U_i > a) \\ 
& = P(\min\{T_i, \dots, T_m\} > a).
\end{aligned}
\end{equation}
Thus, $U_i - U_{i-1} \sim \mathrm{Exponential} ((m-i+1)\lambda)$ from the first property.

Then we have:
\begin{equation}
\nonumber
\begin{aligned}
\mathbb{E}[U_k] &=  \mathbb{E}[\sum_{i=2}^{k}(U_i - U_{i-1}) + U_1] = \sum_{i=2}^{k} \mathbb{E}[(U_i - U_{i-1})] + \mathbb{E}[U_1] = \frac{1}{\lambda}\sum_{x=m-k+1}^{m}\frac{1}{x}.
\end{aligned}
\end{equation}
Since $\frac{1}{x}$ is convex, we know from the lower and upper Riemann sum that: 
\begin{equation}
\nonumber
\int_{m-k+1}^{m+1}\frac{1}{x}dx\leq \sum_{x=m-k+1}^{m}\frac{1}{x} \leq \int_{m-k}^{m}\frac{1}{x}dx
\end{equation}
Then the lower bound can be derived as follows:
\begin{equation}
\nonumber
\begin{aligned}
\int_{m-k+1}^{m+1}\frac{1}{x}dx &= \ln{\frac{m+1}{m-k+1}} \\ 
&\geq 1 - \frac{m-k+1}{m+1} \\
&=\frac{k}{m+1} = \frac{C-Npq}{N(1-p) + 1},
\end{aligned}
\end{equation}
where the first inequality comes from the fact that $1-\frac{1}{x}\leq \ln{x}\leq x-1$. 

Similarly, the upper bound can be derived as follows:
\begin{equation}
\nonumber
\begin{aligned}
\int_{m-k}^{m}\frac{1}{x}dx&=\ln\frac{m}{m-k}\\
&\leq \frac{m}{m-k} - 1 \\ 
& = \frac{k}{m-k} \\ 
& = \frac{C-Npq}{N(1-p) - (C-Npq)} \\
& = \frac{C}{N\frac{(1-p)}{(1-\frac{Nq}{C}p)} - C} \\ 
& \leq \frac{C}{N-C},
\end{aligned}
\end{equation}
where the last inequality comes from the fact that server tends to oversample $q \geq \frac{C}{N}$.  
\end{proof}

\end{document}